%% file: main.tex
\documentclass[twoside,11pt]{article}

\usepackage{jmlr2e}
\input{notation.tex}

\jmlrheading{}{2020}{}{}{}{Rafael Hanashiro and Jacob Abernethy}

\ShortHeadings{Linear Separation via Optimism}{Hanashiro and Abernethy}
\firstpageno{1}

\begin{document}

\title{Linear Separation via Optimism}

\author{\name Rafael Hanashiro \email rhanashiro3@gatech.edu \\
       \addr College of Computing, Georgia Institute of Technology
       \AND
       \name Jacob Abernethy \email prof@gatech.edu \\
       \addr College of Computing, Georgia Institute of Technology}

% \editor{}

\maketitle

\begin{abstract}%   <- trailing '%' for backward compatibility of .sty file
Binary linear classification has been explored since the very early days of the machine learning literature. Perhaps the most classical algorithm is the Perceptron (\cite{perceptron}), where a weight vector used to classify examples is maintained, and additive updates are made as incorrect examples are discovered. The Perceptron has been thoroughly studied and several versions have been proposed over many decades. The key theoretical fact about the Perceptron is that, so long as a perfect linear classifier exists with some margin $\gamma > 0$, the number of required updates to find such a perfect linear separator is bounded by $\frac 1 {\gamma^2}$. What has never been fully addressed is: does there exist an algorithm that can achieve this with fewer updates? In this paper we answer this in the affirmative: we propose the Optimistic Perceptron  algorithm, a simple procedure that finds a separating hyperplane in no more than $\frac 1 \gamma$ updates. We also show experimentally that this procedure can significantly outperform Perceptron.
\end{abstract}


\section{Introduction}
The problem of linear classification has an extensive history that dates back to the mid-20th century. One of the earliest both practical and theoretical accomplishments in machine learning is known as the Perceptron algorithm, introduced by \cite{perceptron}. Many variants have been given over subsequent decades, including the Winnow algorithm (\cite{winnow}), the Voted Perceptron (\cite{voted_perceptron}) and the Second Order Perceptron (\cite{2order_perceptron}), which exploit assumptions such as sparse solutions or second-order information. The literature has also extended to robust methods in the non-separable setting; e.g., Kernel Perceptron (\cite{kernel_perceptron}) and the Maxover algorithm (\cite{maxover}). Despite the long history of development in the field, there has been little progress in substantially improving the fundamental Perceptron mistake bound presented by \cite{perceptron_bound}. In this work, we propose the Optimistic Perceptron  algorithm, a simple procedure for learning a linear predictor that incorporates optimism in its updates. Given that we have $n$ labelled data points, and we assume there exists a perfect linear separator with margin threshold $\gamma > 0$, we claim that Optimistic Perceptron  converges after a total of $O(\frac{1}{\gamma})$ updates, a significant improvement over the Perceptron bound of $O(\frac 1 {\gamma^2})$. While each update requires a loop over $n$ examples, the same is effectively true for the Perceptron, which requires searching through all examples to find a violation.

Our result ostensibly beats a well-known lower bound, as it is easy to construct scenarios in which any online algorithm that can only consider one example on each update must perform $O(\frac{1}{\gamma^2})$ updates in order to find a perfect linear separator. But our method uses two non-standard tools to achieve a faster rate. First, at each round we construct a ``pseudoexample'' by taking a weighted average of existing examples. Second, our algorithm performs an ``optimistic'' update on the weight vector, following an idea put forward by \cite{omd_omd} and originally presented by \cite{og_opt}.

\begin{named}{Outline}
In Section~\ref{sec:main_result}, we introduce basic notation along with the general problem statement and describe the Optimistic Perceptron  algorithm. In Section~\ref{sec:experiments}, we compare the empirical performance between our algorithm and Perceptron on a small margin example. In Section~\ref{sec:motivation}, we detail the game theoretic motivation of the algorithm and state the regret bound theorems and convergence rates.
\end{named}

\section{Main Result}
\label{sec:main_result}

\begin{named}{Notation}
Throughout this paper, $\norm{\cdot}$ will denote the Euclidean norm. More generally, for $p\geq1$ we denote the $\ell_p$-norm by $\norm{\cdot}_p$. We use lowercase and uppercase boldface letters to describe vectors and matrices respectively. For a dataset of size $n$, we consider inputs in $\X=\Rd$ with labels in $\Y=\cbr{\pm1}$ and we denote the $i$th example pair by $(\vecx{i}, \scaly{i})\in\X\times\Y$. For sequentially updated variables, we indicate the iteration with a subscript (e.g., time $t$ iterate $\vecp{t}$). We use $v_j$ to denote the $j$th coordinate of a vector $\mathbf{v}$ and $\indp{t}{j}$ to denote the $j$th coordinate of iterate $\vecp{t}$. For element-wise multiplication between two vectors, we use $\odot$. Given an integer $n\in\N$, we define set $[n]=\cbr{1,\dots,n}$ and probability simplex $\Delta_n=\cbr{\bfp\in\R^n: \sumtn{i=1}{n}p_i=1,\, p_i\geq0 \text{ for } i\in[n]}$.
\end{named}

\subsection{Problem statement}
Define dataset $S=\cbr{(\vecx{i},\scaly{i})}_{i=1}^n\subseteq\X\times\Y$ and let $r>0$ be an upper bound on $\cbr{\norm{\vecx{i}}}_{i=1}^n$. We say that the data are \emph{linearly separable} if there is some $\bfw^*\in\R^d$ satisfying $\scaly{i}\inner{\bfw^*,\vecx{i}} > 0$ for every $i \in [n]$. We say that the data are linearly separable \emph{with margin} $\gamma > 0$ when $\scaly{i}\inner{\bfw^*,\vecx{i}}\geq \gamma$ for every $i\in[n]$. Given a linearly separable dataset $S$ with margin $\gamma$, our goal is to find any linear separator $\bfw$ in the fewest number of updates as a function of $\gamma$ and $n$. For convenience, we introduce the following notation:
\begin{equation*}
    \mathbf{X} =
    \begin{bmatrix}
        \horzbar & \vecx{1} & \horzbar\\
        & \vdots &\\
        \horzbar & \vecx{n} & \horzbar
    \end{bmatrix}
    \quad\text{and}\quad
    \mathbf{y} =
    \begin{bmatrix}
        \scaly{1}\\
        \vdots\\
        \scaly{n}
    \end{bmatrix}
\end{equation*}
In what follows, we define a unit of operation as an addition or inner product between two vectors in $\Rd$ and define an iteration as a single operation.

\subsection{Optimistic Perceptron }
We propose an algorithm that finds a separating hyperplane in an iterative fashion. On each round, we first update the hyperplane and subsequently form a distribution over the data. The procedure is described in Algorithm~\ref{alg:opt_sep}.
\begin{algorithm}[H]
\caption{Optimistic Perceptron }
\label{alg:opt_sep}
\begin{algorithmic}[1]
\State $\vecp{0} \gets\paren{\frac{1}{n},\dots,\frac{1}{n}}$
\State $\vectx{-1}, \vectx{0} \gets (\vecp{0} \odot \mathbf{y})^\top\mathbf{X}$ \Comment{init pseudoexamples}
\State $\vecw{0} \gets \bm{0}$
\For{$t\gets1$ \textbf{to} $T$}
    \State $\vecw{t} \gets \vecw{t-1} + 2\vectx{t-1} - \vectx{t-2}$ \Comment{optimistic linear update}
    % \State $\vecw{t} \gets \vecw{t-1} + \frac{1}{2}\brac{\paren{2\vecp{t-1}-\vecp{t-2}}\odot\mathbf{y}}^\top\mathbf{X}$
    \For{$i\gets1$ \textbf{to} $n$}
        \State $\indp{t}{i} \gets \frac{\indp{t-1}{i} \exp\paren{-\frac{1}{r^2}\scaly{i}\vecw{t}^\top\vecx{i}} }{\sumtn{j=1}{n} \indp{t-1}{j} \exp\paren{-\frac{1}{r^2}\scaly{j}\vecw{t}^\top\vecx{j}}}$ \Comment{update example weights}
    \EndFor
    \State $\vectx{t} \gets (\vecp{t} \odot \mathbf{y})^\top \mathbf{X}$ \Comment{construct $t$th pseudoexample}
\EndFor
\State Return $\frac{1}{T}\sumtn{t=1}{T}\vecw{t}$
\end{algorithmic}
\end{algorithm}
\noindent For a separating hyperplane $\bfw^*$, we define its margin as
\begin{equation}
\label{eq:perceptron_margin}
    \gamma = \min_{i\in[n]} \frac{\abs{\inner{\bfw^*, \vecx{i}}}}{\norm{\bfw^*}}
\end{equation}
\cite{perceptron_bound} showed that the Perceptron algorithm converges after making $O(\frac{1}{\gamma^2})$ \emph{updates}. What is often overlooked is the fact that, in the batch setting, Perceptron can take $O(n)$ time to find a violating example, implying that as many as $O(\frac{n}{\gamma^2})$ iterations are required in order to find a perfect separator. We claim that Algorithm~\ref{alg:opt_sep} converges after $O(\frac{n}{\gamma})$ iterations. The details are given in Section~\ref{sec:motivation}.

\section{Experiments}
\label{sec:experiments}
Consider the dataset of the form
\begin{equation}
\label{eq:ex_dataset}
    \vecx{i} = (\underbrace{(-1)^i,\dots,(-1)^i,(-1)^{i+1}}_{i\text{ first components}},0,\dots,0)\in\R^n, \quad \scaly{i}=(-1)^{i+1}
\end{equation}
for $i=1,\dots,n$. This is a well-known example in which Perceptron makes $\Omega(2^n)$ mistakes. We ran both algorithms on this dataset with $n$ ranging from 1 to 15. For Perceptron, we repeatedly passed through the data, iterating in the order of the points on each pass (i.e., from $\vecx{1}$ to $\vecx{n}$). Figure~\ref{fig:ex} illustrates the \emph{total} number of iterations needed by each algorithm before reaching a separating hyperplane. An important observation, and possible explanation for the significant discrepancy in performance, is that the average number of mistakes made by Perceptron on each pass through the data remained approximately constant at 2 as we increased the size of the dataset. In such cases, the upper bound on the total iteration count required by Perceptron is $O(\frac{n}{\gamma^2})$, and our rate is an improvement of $\frac{1}{\gamma}$.

\begin{figure}
    \centering
    \begin{subfigure}[b]{0.5\textwidth}
        \centering
        \includegraphics[width=\textwidth]{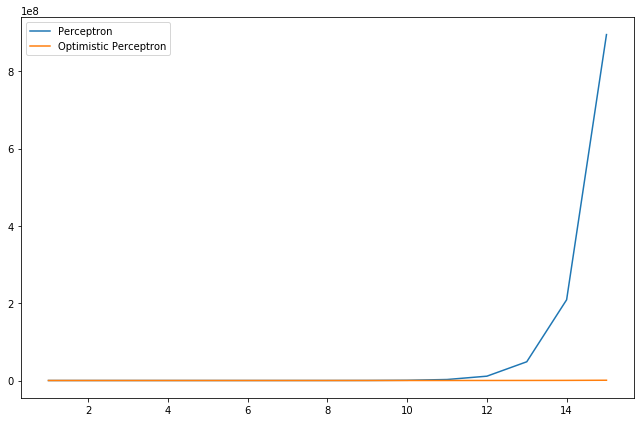}
        \caption{total iteration count vs. $n$}
        \label{fig:ex_count}
    \end{subfigure}%
    \begin{subfigure}[b]{0.5\textwidth}
        \centering
        \includegraphics[width=\textwidth]{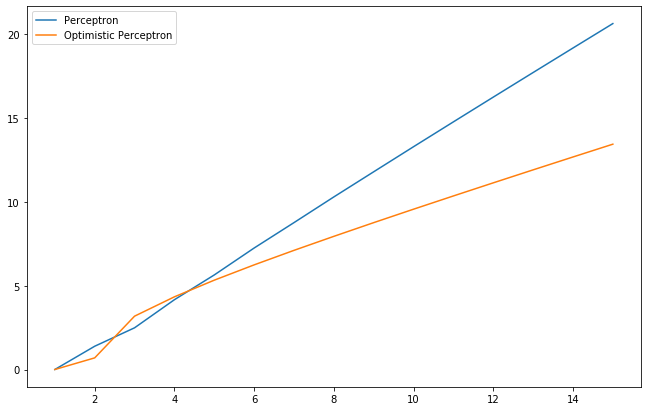}
        \caption{log(total iteration count) vs. $n$}
        \label{fig:ex_logcount}
    \end{subfigure}
    \caption{Number of total iterations needed by the Perceptron and Optimistic Perceptron  algorithms on dataset \eqref{eq:ex_dataset} until convergence. (\subref{fig:ex_count}) displays the iteration count as we increase $n$, and (\subref{fig:ex_logcount}) displays the logarithmic relationship.}
    \label{fig:ex}
\end{figure}

\section{Motivation - Optimism Under Uncertainty}
\label{sec:motivation}
Algorithm~\ref{alg:opt_sep} is inspired by a game theoretic interpretation of the linear separation problem. We begin by introducing a new formulation of Equation~\eqref{eq:perceptron_margin} and then construct a game framework under which we approximate an equilibrium point via optimistic online strategies. The approach is largely inspired by the framework to minimize smooth functions proposed by \cite{oftl_md}. In Appendix~\ref{sec:appendix_A}, we show how the prescribed strategy is equivalent to Algorithm~\ref{alg:opt_sep} for a particular choice of hyperparameters, and regret bound proofs are deferred to Appendix~\ref{sec:appendix_B}.

\subsection{Linear separation as a game}
We begin by assuming that $S$ can be linearly separated by a vector $\bfw^*$, and we define unit vector $\bfwsh = \frac{\bfw^*}{\norm{\bfw^*}}$. Equation~\eqref{eq:perceptron_margin} implies that $\scaly{i}\inner{\bfwsh,\vecx{i}} \geq \gamma$ for every $i\in[n]$. If we define functions $g_i(\bfw)=\scaly{i}\bfw^\top\vecx{i}$, then for any distribution $\bfp\in\Delta_n$ over the data, we have that $\sumtn{i=1}{n}p_ig_i(\bfwsh) \geq \sumtn{i=1}{n}p_i\gamma = \gamma$. We can restate this as the following bound:
\begin{equation}
\label{eq:minimax_exp}
    \min_{\bfp\in\Delta_n} \bfp^\top\Mwsh \geq \gamma
\end{equation}
where we define vector $\Mw=\brac{g_1(\bfw),\dots,g_n(\bfw)}^\top$.

\subsection{Approximating equilibria}
We can now view the problem as a zero-sum game between two players: the data and the learner. The former aims to minimize objective $\bfp^\top\Mw$ over $\bfp\in\Delta_n$, while the latter aims to maximize it over $\bfw\in\Rd$. We will obtain an approximate equilibrium to the game through sequential playing: the learner first selects a vector using an Optimistic Follow the Regularized Leader (OFTRL) strategy and the data subsequently responds with a distribution via Mirror Descent (MD).

We begin by characterizing the learner's strategy and derive a bound on its total regret. Note that the learner plays first and, thus, only has access to the data's actions up until the previous round. Theorem~\ref{thm:p1_regret} describes the regret bound.
\begin{theorem}
\label{thm:p1_regret}
Suppose the learner plays first and uses an OFTRL strategy. That is, it makes the following decision on round $t$:
\begin{equation*}
    \vecw{t} = \argmax_{\bfw\in\Rd} \cbr{\paren{\vecp{t-1}+\sumtn{s=1}{t-1}\vecp{s}}^\top\Mw - \frac{1}{2}\norm{\bfw}^2}
\end{equation*}
Then we can bound its regret at time $T$ by
\begin{equation}
\label{eq:p1_regret}
    \paren{\sumtn{t=1}{T}\vecp{t}}^\top\Mwsh - \sumtn{t=1}{T}\vecp{t}^\top\Mwt \leq \frac{1}{2} + \sumtn{t=1}{T} \frac{r^2}{2} \norm{\vecp{t}-\vecp{t-1}}_1^2
\end{equation}
\end{theorem}
\noindent Through the use of optimism, we can have the data counter the summation on the right-hand side of \eqref{eq:p1_regret}, resulting in a constant value when both bounds are added together. Note that the data plays second, so it has knowledge of the learner's action on each round. Theorem~\ref{thm:p2_regret} showcases this idea in more detail.
\begin{theorem}
\label{thm:p2_regret}
Let $\Rcal:\Delta_n\to\R$ be a $\lambda$-strongly convex function with respect to $\norm{\cdot}_1$ and let $H$ be an upper bound on $\breg\paren{\cdot,\cdot}$. Suppose that the data updates according to MD:
\begin{equation*}
    \vecp{t} = \argmin_{\bfp\in\Delta_n} \eta_t\inner{\bfp,\Mwt} + \breg\paren{\bfp,\vecp{t-1}}
\end{equation*}
for a non-increasing sequence $\cbr{\eta_t}_{t=1}^T$. Then we can describe the following bound on its regret at time $T$:
\begin{equation}
\label{eq:p2_regret}
    \sumtn{t=1}{T} \vecp{t}^\top\Mwt - \min_{\bfp\in\Delta_n} \bfp^\top\paren{\sumtn{t=1}{T}\Mwt} \leq \frac{H}{\eta_T} - \sumtn{t=1}{T} \frac{\lambda}{2\eta_t} \norm{\vecp{t}-\vecp{t-1}}_1^2
\end{equation}
\end{theorem}
\noindent We can combine Theorems~\ref{thm:p1_regret} and \ref{thm:p2_regret} to get the following result.
\begin{theorem}
\label{thm:comb_regret}
If we set $\eta_t=\frac{\lambda}{r^2}$ for all $t\in[T]$ in the MD update, then the following bound holds true:
\begin{equation}
    \min_{\bfp\in\Delta_n} \bfp^\top\Mwsh - \min_{\bfp\in\Delta_n} \bfp^\top\paren{\frac{1}{T}\sumtn{t=1}{T}\Mwt} \leq \frac{\lambda+2Hr^2}{2\lambda T}
\end{equation}
\end{theorem}
\begin{proof}
Adding bounds \eqref{eq:p1_regret} and \eqref{eq:p2_regret} yields
\begin{align*}
    \min_{\bfp\in\Delta_n} \bfp^\top\Mwsh - \min_{\bfp\in\Delta_n} \bfp^\top\paren{\frac{1}{T}\sumtn{t=1}{T}\Mwt} &\leq   \paren{\frac{1}{T}\sumtn{t=1}{T}\vecp{t}}^\top\Mwsh - \min_{\bfp\in\Delta_n} \bfp^\top\paren{\frac{1}{T}\sumtn{t=1}{T}\Mwt}\\
    &\leq \frac{\eta_T + 2H}{2\eta_T T} + \frac{1}{T}\sumtn{t=1}{T} \paren{\frac{r^2}{2}-\frac{\lambda}{2\eta_t}} \norm{\vecp{t}-\vecp{t-1}}_1^2
\end{align*}
By setting $\eta_t=\frac{\lambda}{r^2}$ for all $t\in[T]$, the summation on the right-hand side of the inequality cancels out. Substituting for $\eta_T$ leads to the desired bound.
\end{proof}
If both players play this game for $T>\frac{\lambda+2Hr^2}{2\lambda\gamma}$ rounds, where $\gamma$ is the margin as defined in \eqref{eq:perceptron_margin}, then Theorem~\ref{thm:comb_regret} and inequality \eqref{eq:minimax_exp} imply that
\begin{align}
    \min_{\bfp\in\Delta_n} \bfp^\top\paren{\frac{1}{T}\sumtn{t=1}{T}\Mwt} &\geq \min_{\bfp\in\Delta_n} \bfp^\top\Mwsh - \frac{\lambda+2Hr^2}{2\lambda T} \notag\\
    &> \min_{\bfp\in\Delta_n} \bfp^\top\Mwsh - \gamma \notag\\
    &\geq 0 \label{eq:minimax_conc}
\end{align}
Let us define $\bar{\bfw}_T=\frac{1}{T}\sum_t\vecw{t}$. Then \eqref{eq:minimax_conc} implies that 
\begin{equation*}
    \scaly{i}\inner{\bar{\bfw}_T,\vecx{i}} = \frac{1}{T}\sumtn{t=1}{T}g_i(\vecw{t}) = \brac{\frac{1}{T}\sumtn{t=1}{T}\Mwt}_i > 0
\end{equation*}
for every $i\in[n]$. In particular, it follows that $\bar{\bfw}_T$ correctly classifies every point in $S$. We were able to achieve this classifier after $O(\frac{1}{\gamma})$ rounds of the game, and a total of $O(\frac{n}{\gamma})$ iterations since each MD update requires passing through the data.

\section{Conclusion}
While the machine learning literature has seen tremendous development in the problem of linear prediction, little has been done to fundamentally improve the Perceptron mistake bound. In this paper, we showed that via game theoretic lens, we can use optimism to obtain a much faster rate with a remarkably simple procedure that performs well in theory and in practice.

\bibliography{references}

\newpage

\appendix

\section{Algorithmic details}
\label{sec:appendix_A}
Here, we derive closed form expressions for the updates in Section~\ref{sec:motivation} and show how they describe Algorithm~\ref{alg:opt_sep}.

\subsection{Learner update}
\newcommand{\bfy}{\mathbf{y}}
\newcommand{\bfX}{\mathbf{X}}
\newcommand{\bfwstar}{\mathbf{w}^*}
\newcommand{\lambdastar}{\lambda^*}

The OFTRL update is of the form
\begin{equation*}
    \argmax_{\bfw\in\Rd} \paren{\bfp\odot\bfy}^\top\bfX\bfw - \frac{1}{2}\norm{\bfw}^2
\end{equation*}
for some vector $\bfp\in\R^n$. We have that solution $\bfwstar$ satisfies $\bfwstar = \paren{\bfp\odot\bfy}^\top\bfX$. Hence, using the proper terms, we get that
\begin{equation*}
    \vecw{t} = \brac{\paren{\vecp{t-1} + \sumtn{s=1}{t-1}\vecp{s}} \odot \bfy}^\top\bfX
\end{equation*}

\subsection{Data update}
\newcommand{\bfq}{\mathbf{q}}
\newcommand{\bfz}{\mathbf{z}}
\newcommand{\bmu}{\bm{\mu}}
\newcommand{\bfpstar}{\mathbf{p}^*}
\newcommand{\bmustar}{\bmu^*}

For the MD update, we will define $\Rc$ to be the negative entropy function, which is 1-strongly convex w.r.t. $\norm{\cdot}_1$ over the simplex. That is, for a distribution $\bfp\in\Delta_n$ we have that
\begin{equation*}
    \Rc(\bfp) = \sumtn{i=1}{n} p_i\log p_i
\end{equation*}
Then the Bregman divergence will coincide with the KL divergence: $\breg\paren{\bfp, \bfq} = \KL\paren{\bfp||\bfq}$. This leads to a convex optimization problem of the form
\begin{align*}
    \min_{\bfp\in\R^n} \;& \bfp^\top\bfz + \sumtn{i=1}{n} p_i\log\frac{p_i}{q_i}\\
    \text{s.t.} \;& \bfp \succeq 0\\
    & \sumtn{i=1}{n} p_i = 1
\end{align*}
We can write the Lagrangian as
\begin{equation*}
    \L\paren{\bfp, \bmu, \lambda} = \bfp^\top\bfz + \sumtn{i=1}{n} p_i\log\frac{p_i}{q_i} - \bmu^\top\bfp + \lambda \paren{\sumtn{i=1}{n}p_i - 1}
\end{equation*}
From strong duality and KKT, we know that primal and dual optimal $\bfpstar$ and $(\bmustar, \lambdastar)$ satisfy
\begin{equation*}
    \brac{\nabla_\bfp \L\paren{\bfpstar, \bmustar, \lambdastar}}_i = z_i + 1 + \log\frac{p^*_i}{q_i} - \mu^*_i + \lambdastar = 0
\end{equation*}
which implies that 
\begin{equation*}
    p^*_i = \frac{q_i \exp(-z_i)}{\exp(1 - \mu^*_i + \lambdastar)}
\end{equation*}
Since $\sum_i p^*_i = 1$, the denominator on the right-hand side will serve as a normalization term. That is, $\exp(1 - \mu^*_i + \lambdastar) = \sum_j q_j \exp(-z_j)$ for all $i\in[n]$, and we can set $\bmustar = 0$ to satisfy dual feasibility and complementary slackness. By setting $\bfz = \eta_t\Mwt$ and $\bfq = \vecp{t-1}$ where $\eta_t = \frac{1}{r^2}$ for all $t$, we get that the update corresponds to an instance of the Exponential Weights Algorithm and that the iterate at time $t$ is defined by
\begin{equation*}
    \indp{t}{i} = \frac{\indp{t-1}{i} \exp\paren{-\frac{1}{r^2} \scaly{i} \vecw{t}^\top\vecx{i}}}{\sumtn{j=1}{n} \indp{t-1}{j} \exp\paren{-\frac{1}{r^2} \scaly{j} \vecw{t}^\top\vecx{j}}}
\end{equation*}

\section{Regret bound proofs}
\label{sec:appendix_B}

\subsection{Proof of Theorem~\ref{thm:p1_regret}}
% f_t(w)
\newcommand{\ftldw}[1]{\tilde{f}_{#1}(\bfw)}
\newcommand{\Ftldw}[1]{\tilde{F}_{#1}(\bfw)}
\newcommand{\fw}[1]{f_{#1}(\bfw)}
\newcommand{\Fw}[1]{F_{#1}(\bfw)}

% f_t(w_t)
\newcommand{\ftldwt}[2]{\tilde{f}_{#1}(\bfw_{#2})}
\newcommand{\Ftldwt}[2]{\tilde{F}_{#1}(\bfw_{#2})}
\newcommand{\fwt}[2]{f_{#1}(\bfw_{#2})}
\newcommand{\Fwt}[2]{F_{#1}(\bfw_{#2})}

% f_t(hat(w)^*)
\newcommand{\ftldwsh}[1]{\tilde{f}_{#1}(\bfwsh)}
\newcommand{\Ftldwsh}[1]{\tilde{F}_{#1}(\bfwsh)}
\newcommand{\fwsh}[1]{f_{#1}(\bfwsh)}
\newcommand{\Fwsh}[1]{F_{#1}(\bfwsh)}

% Other
\newcommand{\Mprime}{\mathbf{M}'}  
\newcommand{\bfu}{\mathbf{u}}
\newcommand{\fprime}[1]{\mathbf{f}_{#1}'}

\begin{proof2}
For this proof, we will define the following functions:
\begin{align*}
    & \fw{t} = \vecp{t}^\top\Mw \\
    & \ftldw{t} = \vecp{t-1}^\top\Mw\\
    & \Fw{t} = \sumtn{s=1}{t-1}\fw{s} - \frac{1}{2}\norm{\bfw}^2\\
    & \Ftldw{t} = \ftldw{t} + \sumtn{s=1}{t-1}\fw{s} - \frac{1}{2}\norm{\bfw}^2 = \ftldw{t} + \Fw{t}
\end{align*}
The OFTRL update can then be written as
\begin{equation*}
    \vecw{t} = \argmax_{\bfw\in\Rd} \cbr{\paren{\vecp{t-1}+\sumtn{s=1}{t-1}\vecp{s}}^\top\Mw - \frac{1}{2}\norm{\bfw}^2} = \argmax_{\bfw\in\Rd} \Ftldw{t}
\end{equation*}
Recall that $\bfwsh$ is a unit vector in $\Rd$. We have that
\begin{align*}
    \paren{\sumtn{t=1}{T}\vecp{t}}^\top\Mwsh - \sumtn{t=1}{T}\vecp{t}^\top\Mwt &= \sumtn{t=1}{T}\fwsh{t} - \sumtn{t=1}{T}\fwt{t}{t}\\
    &= \frac{1}{2}\norm{\bfwsh}^2 - \ftldwsh{T+1} + \Ftldwsh{T+1} - \sumtn{t=1}{T}\fwt{t}{t}\\
    &= \frac{1}{2} - \ftldwsh{T+1} + \Ftldwsh{T+1} - \Ftldwt{T+1}{T+1} + \Ftldwt{1}{1}\\
    &\hspace{3.5cm} + \sumtn{t=1}{T} \brac{\Ftldwt{t+1}{t+1} - \Ftldwt{t}{t} - \fwt{t}{t}}\\
    % &\leq \frac{1}{2} - \ftldwsh{T+1} + \Ftldwt{1}{1}\\
    % &\hspace{3.5cm} + \sumtn{t=1}{T} \brac{\Ftldwt{t+1}{t+1} - \Ftldwt{t}{t} - \fwt{t}{t}}\\
    % &= \frac{1}{2} - \ftldwsh{T+1} - \frac{1}{2}\norm{\vecw{1}}^2 + \ftldwt{1}{1}\\
    % &\hspace{0.5cm} + \sumtn{t=1}{T} \brac{\Fwt{t+1}{t+1} - \Fwt{t}{t} - \fwt{t}{t} + \ftldwt{t+1}{t+1} - \ftldwt{t}{t}}\\
    % &\leq \frac{1}{2} - \ftldwsh{T+1} + \ftldwt{T+1}{T+1}\\
    % &\hspace{3.5cm} + \sumtn{t=1}{T} \brac{\Fwt{t+1}{t+1} - \Fwt{t}{t} - \fwt{t}{t}}
\end{align*}
Using the maximality of $\vecw{t}$, we know that $\Ftldwsh{T+1} \leq \Ftldwt{T+1}{T+1}$. Thus,
\begin{align*}
    \paren{\sumtn{t=1}{T}\vecp{t}}^\top\Mwsh - \sumtn{t=1}{T}\vecp{t}^\top\Mwt &\leq \frac{1}{2} - \ftldwsh{T+1} + \Ftldwt{1}{1} + \sumtn{t=1}{T} \brac{\Ftldwt{t+1}{t+1} - \Ftldwt{t}{t} - \fwt{t}{t}}\\
    &= \frac{1}{2} - \ftldwsh{T+1} - \frac{1}{2}\norm{\vecw{1}}^2 + \ftldwt{1}{1}\\
    &\hspace{0.5cm} + \sumtn{t=1}{T} \brac{\Fwt{t+1}{t+1} - \Fwt{t}{t} - \fwt{t}{t} + \ftldwt{t+1}{t+1} - \ftldwt{t}{t}}\\
    &\leq \frac{1}{2} - \ftldwsh{T+1} + \ftldwt{T+1}{T+1} + \sumtn{t=1}{T} \brac{\Fwt{t+1}{t+1} - \Fwt{t}{t} - \fwt{t}{t}}
\end{align*}
Since the function $\tilde{f}_{T+1}$ won't affect the rest of the analysis, we will set it to the zero map. Then the bound above becomes
\begin{equation}
\label{eq:reg_bound_1}
    \paren{\sumtn{t=1}{T}\vecp{t}}^\top\Mwsh - \sumtn{t=1}{T}\vecp{t}^\top\Mwt \leq \frac{1}{2} + \sumtn{t=1}{T} \brac{\Fwt{t+1}{t+1} - \Fwt{t}{t} - \fwt{t}{t}}
\end{equation}
Now, let's take a closer look at each term in the summation on the right-hand side:
\begin{equation}
\label{eq:sum_term}
    \Fwt{t+1}{t+1} - \Fwt{t}{t} - \fwt{t}{t} = \Fwt{t}{t+1} + \fwt{t}{t+1} - \paren{\Fwt{t}{t} + \fwt{t}{t}}
\end{equation}
Note that the function $F_t+f_t$ is 1-strongly concave w.r.t. $\norm{\cdot}$ over $\R^d$, and that by the definition of $\vecw{t}$ it follows that $\nabla\ftldwt{t}{t}+\nabla\Fwt{t}{t} = \nabla\Ftldwt{t}{t} = 0$. Hence, in light of \eqref{eq:sum_term}, we get that
\begin{align}
    \Fwt{t+1}{t+1} - \Fwt{t}{t} - \fwt{t}{t} &\leq \inner{\nabla\Fwt{t}{t}+\nabla\fwt{t}{t}, \vecw{t+1}-\vecw{t}} - \frac{1}{2}\norm{\vecw{t+1}-\vecw{t}}^2 \notag\\
    &= \inner{\nabla\fwt{t}{t} - \nabla\ftldwt{t}{t}, \vecw{t+1}-\vecw{t}} - \frac{1}{2}\norm{\vecw{t+1}-\vecw{t}}^2 \notag\\
    &\leq \norm{\nabla\fwt{t}{t} - \nabla\ftldwt{t}{t}} \norm{\vecw{t+1}-\vecw{t}} - \frac{1}{2}\norm{\vecw{t+1}-\vecw{t}}^2 \notag\\
    &\leq \frac{1}{2}\norm{\nabla\fwt{t}{t} - \nabla\ftldwt{t}{t}}^2 \label{eq:grad_norm_bound}
\end{align}
Next, observe that for any $\bfp\in\R^n$, we have that
\begin{equation}
\label{eq:gen_grad}
    \nabla_{\bfw}\paren{\bfp^\top\Mw} = \sumtn{i=1}{n}p_i\scaly{i}\vecx{i} = \bfp^\top\Mprime
\end{equation}
where we define
\begin{equation*}
    \Mprime =
    \begin{bmatrix}
        \vspace{2mm}
        \horzbar & \scaly{1}\vecx{1} & \horzbar\\\vspace{2mm}
        & \vdots &\\
        \horzbar & \scaly{n}\vecx{n} & \horzbar
    \end{bmatrix}
\end{equation*}
Using the definitions of $f_t$ and $\tilde{f}_t$ and the differentiation shown in \eqref{eq:gen_grad}, we can further bound the expression on the right-hand side of \eqref{eq:grad_norm_bound} with
\begin{align*}
    \norm{\nabla\fwt{t}{t} - \nabla \ftldwt{t}{t}}^2 &= \norm{\paren{\vecp{t} - \vecp{t-1}}^\top\Mprime}^2 \\
    &= \norm{\sumtn{i=1}{n} \brac{\vecp{t} - \vecp{t-1}}_i \scaly{i}\vecx{i}}^2 \\
    &\leq \paren{\sumtn{i=1}{n} \abs{\brac{\vecp{t} - \vecp{t-1}}_i} \norm{\vecx{i}}}^2\\
    &\leq \paren{r \sumtn{i=1}{n} \abs{\brac{\vecp{t} - \vecp{t-1}}_i}}^2\\
    &= r^2\norm{\vecp{t} - \vecp{t-1}}_1^2
\end{align*}
where the last inequality follows from our definition of $r$ as an upper bound on the norm of the $\vecx{i}$'s. Combining this result with inequalities \eqref{eq:reg_bound_1} and \eqref{eq:grad_norm_bound} yields
\begin{equation*}
    \paren{\sumtn{t=1}{T}\vecp{t}}^\top\Mwsh - \sumtn{t=1}{T}\vecp{t}^\top\Mwt \leq \frac{1}{2} + \sumtn{t=1}{T} \frac{r^2}{2} \norm{\vecp{t} - \vecp{t-1}}_1^2
\end{equation*}
\end{proof2}

\subsection{Proof of Theorem~\ref{thm:p2_regret}}
\begin{proof2}
% Define variables
\newcommand{\bfr}{\mathbf{r}}

Let us present the MD update again here for convenience:
\begin{equation*}
    \vecp{t} = \argmin_{\bfp\in\Delta_n} \eta_t\inner{\bfp, \Mwt} + \breg\paren{\bfp, \vecp{t-1}}
\end{equation*}
Let $\bfpstar, \bfq, \bfr\in\Delta_n$ denote arbitrary vectors and let us define function $h(\bfp) = \bfp^\top\bfr + \breg\paren{\bfp,\bfq}$. Recall that $\breg\paren{\bfp, \bfq} = \Rc(\bfp) - \Rc(\bfq) - \inner{\nabla\Rc(\bfq), \bfp-\bfq}$, so we have that $\nabla h(\bfp) = \bfr + \nabla\Rc(\bfp) - \nabla\Rc(\bfq)$. Let $\bfp = \argmin_{\bfp'\in\Delta_n} h(\bfp')$. By the first order optimality condition, we have that
\begin{align*}
    \paren{\bfp - \bfpstar}^\top\nabla h(\bfp) \leq 0 &\LA \paren{\bfp - \bfpstar}^\top\paren{\bfr + \nabla\Rc(\bfp) - \nabla\Rc(\bfq)} \leq 0 \\
    &\LA \paren{\bfp - \bfpstar}^\top\bfr \leq \paren{\bfpstar - \bfp}^\top\paren{\nabla\Rc(\bfp) - \nabla\Rc(\bfq)} \\
    &\LA \paren{\bfp - \bfpstar}^\top\bfr \leq \breg\paren{\bfpstar, \bfq} - \breg\paren{\bfpstar, \bfp} - \breg\paren{\bfp, \bfq}
\end{align*}
If we let $\bfr = \eta_t\Mwt$ and $\bfq = \vecp{t-1}$, then $\bfp = \vecp{t}$ by definition, and we get that
\begin{align*}
    \eta_t \paren{\vecp{t} - \bfpstar}^\top\Mwt &\leq \breg\paren{\bfpstar, \vecp{t-1}} - \breg\paren{\bfpstar, \vecp{t}} - \breg\paren{\vecp{t}, \vecp{t-1}} \\
    &\leq \breg\paren{\bfpstar, \vecp{t-1}} - \breg\paren{\bfpstar, \vecp{t}} - \frac{\lambda}{2}\norm{\vecp{t} - \vecp{t-1}}_1^2
\end{align*}
where the last line follows from the $\lambda$-strong convexity of $\Rc$. Summing both sides of the inequality from $t=1,\dots,T$, and noting that $\cbr{\eta_t}_{t=1}^\top$ is a non-increasing sequence, yields
\begin{align*}
    \sumtn{t=1}{T} \paren{\vecp{t} - \bfpstar}^\top\Mwt &\leq \sumtn{t=1}{T} \frac{1}{\eta_t} \brac{\breg\paren{\bfpstar, \vecp{t-1}} - \breg\paren{\bfpstar, \vecp{t}} - \frac{\lambda\norm{\vecp{t} - \vecp{t-1}}_1^2}{2}} \\
    &= \frac{\breg\paren{\bfpstar, \vecp{0}}}{\eta_1} - \frac{\breg\paren{\bfpstar, \vecp{T}}}{\eta_T} \\
    &\hspace{2cm} + \sumtn{t=1}{T-1} \breg\paren{\bfpstar, \vecp{t}}\paren{\frac{1}{\eta_{t+1}} - \frac{1}{\eta_t}} - \sumtn{t=1}{T} \frac{\lambda\norm{\vecp{t} - \vecp{t-1}}_1^2}{2\eta_t} \\
    &\leq H\paren{\frac{1}{\eta_1} + \frac{1}{\eta_T} - \frac{1}{\eta_1}} - \sumtn{t=1}{T} \frac{\lambda\norm{\vecp{t} - \vecp{t-1}}_1^2}{2\eta_t} \\
    &\leq \frac{H}{\eta_T} - \sumtn{t=1}{T} \frac{\lambda\norm{\vecp{t} - \vecp{t-1}}_1^2}{2\eta_t}
\end{align*}
where $H$ is an upper bound on $\breg\paren{\cdot,\cdot}$.
\end{proof2}

\end{document}

%% file: notation.tex
\DeclareMathOperator*{\argmax}{argmax}
\DeclareMathOperator*{\argmin}{argmin}

\newcommand{\N}{\mathbb{N}}

\newcommand{\R}{\mathbb{R}}

\newcommand{\Rd}{\R^d}

\newcommand{\D}{\mathcal{D}}

\renewcommand{\L}{\mathcal{L}}

\newcommand{\X}{\mathcal{X}}
\newcommand{\Y}{\mathcal{Y}}
\newcommand{\Rcal}{\mathcal{R}}

\newcommand{\inner}[1]{\left\langle #1\right\rangle}
\newcommand{\norm}[1]{\left\lVert#1\right\rVert}

\newcommand{\vecx}[1]{\mathbf{x}^{(#1)}}
\newcommand{\vectx}[1]{\tilde{\mathbf{x}}^{[#1]}}
\newcommand{\scaly}[1]{y^{(#1)}}

\newcommand{\vecp}[1]{\mathbf{p}_{#1}}
\newcommand{\indp}[2]{p_{#1,#2}}
\newcommand{\vecw}[1]{\mathbf{w}_{#1}}

\newcommand{\bfw}{\mathbf{w}}
\newcommand{\bfp}{\mathbf{p}}
\newcommand{\Mw}{\mathbf{M}_\bfw}
\newcommand{\Mwt}{\mathbf{M}_{\vecw{t}}}

\newcommand{\bfwsh}{\hat{\bfw}^*}
\newcommand{\Mwsh}{\mathbf{M}_{\bfwsh}}

\newcommand{\paren}[1]{\left(#1\right)}
\newcommand{\brac}[1]{\left[#1\right]}
\newcommand{\cbr}[1]{\left\{#1\right\}}
\newcommand{\abs}[1]{\left|#1\right|}

\newcommand{\LA}{\Longrightarrow}

\newcommand{\sumtn}[2]{\sum_{#1}^{#2}}

\newcommand{\Rc}{\mathcal{R}}
\newcommand{\breg}{\D_\Rc}

\newcommand{\KL}{\text{KL}}

\newcommand*{\horzbar}{\rule[.5ex]{3ex}{0.5pt}}